\documentclass[11pt,a4paper,twoside]{article}
\usepackage[T1]{fontenc}
\usepackage[margin=2.8cm,headheight=14pt,headsep=18pt]{geometry}
\usepackage{amsmath,amssymb,amsthm,mathtools}
\usepackage{graphicx,booktabs,array,microtype,fancyhdr}
\usepackage[font=small,labelfont=bf,hypcap=false]{caption}
\usepackage{enumitem}
\usepackage{xurl}
\usepackage[hidelinks]{hyperref}
\hypersetup{pdftitle={Marginal Log-Likelihood Increments under Dirichlet-Smoothed Markov Estimation},pdfauthor={Levin David Schwab},pdfcreator={},pdfproducer={},pdfsubject={},pdfkeywords={}}
\setlist{nosep,leftmargin=1.5em}
\theoremstyle{plain}
\newtheorem{proposition}{Proposition}[section]

\newtheorem{corollary}[proposition]{Corollary}
\theoremstyle{definition}

\numberwithin{equation}{section}
\newcommand{\E}{\mathbb E}
\newcommand{\Var}{\operatorname{Var}}
\newcommand{\Cov}{\operatorname{Cov}}
\newcommand{\TV}{\operatorname{TV}}
\newcommand{\osc}{\operatorname{osc}}
\newcommand{\KL}{D}
\newcommand{\ind}{\mathbf 1}
\newcommand{\R}{\mathbb R}
\newcommand{\note}[1]{\par\smallskip\noindent #1\par\smallskip}
\begin{document}
\raggedbottom
\begin{center}
{\Large\bfseries Marginal Log-Likelihood Increments\\[2pt] under Dirichlet-Smoothed Markov Estimation\par}
\vspace{6pt}
{\large Exact trace valuation, attainable gain, and budgeted selection\par}
\vspace{12pt}
{\large Levin David Schwab\par}
\vspace{4pt}
21 September 2026
\end{center}
\thispagestyle{plain}
\vspace{3pt}
\begin{abstract}
\noindent For a Dirichlet-smoothed transition model, the effect of adding one workflow trace to the training archive is an exact change in reference-weighted log likelihood. We derive that change and show that it is a weighted reduction of Kullback--Leibler divergence between the reference conditionals and the model. From this form we obtain an upper bound on the gain available to any acquisition, which expresses a millinat difference as a share of what is attainable, an exact covariance identity for the effect of the reference weighting, and a sign criterion for the interaction between two candidates, from which the batch objective is neither submodular nor supermodular. A case study on the BPI Challenge 2012 loan-application log measures all three and finds a positive selection result in one of the four combinations of reference weighting and budget unit. There, of two regressors fitted to identical descriptors and identical labels, the one that predicts individual increments far more accurately, median $R^2$ 0.87 against 0.62, realizes the smaller share of the attainable gain, 61 against 69 per cent, so ranking accuracy for individual traces is neither necessary nor sufficient for batch quality.
\end{abstract}

\section{Introduction}
A workflow trace records the activities of one recorded case, such as a loan application. Suppose an archive already trains a model to predict the next recorded activity. An additional trace is useful if adding it reduces prediction loss on the workload against which the model is evaluated.

The practical question is whether inexpensive descriptions of a trace can identify useful additions under a limited acquisition budget. A useful answer must distinguish three claims. First, a descriptor may predict a trace's marginal contribution for a fixed archive and evaluation population. Second, that prediction may remain useful after the evaluation period changes. Third, selecting a batch by individual predictions may improve the model. A priori none of these claims implies the next.

The learner studied here retains only the current activity and is estimated by counting transitions. The simplicity is deliberate: it makes every marginal contribution exactly calculable and removes valuation noise as an explanation for failure. It also limits the conclusions, since a first-order learner represents no long-range dependence.

\section{Background}\label{sec:background}
\emph{Event logs.} An event log is a table produced as a by-product of a business information system. Each row records that a named activity occurred in a named case at a given time. Grouping the rows by case and ordering them by time turns the log into a finite set of finite sequences over a finite alphabet of activities, and one such sequence is a trace. The log used here has 164,506 retained events in 13,087 cases over an alphabet of 23 activities. A log is an observed record rather than a sample from a postulated process, and the only structure used below is the ordering of activities within a case. Selecting a subset of such a log for training has been studied under other objectives \cite{fanisani2023}.

\emph{The learner.} The prediction task is the next recorded activity given the current one, one of the tasks collected under predictive process monitoring \cite{teinemaa2019}. A first-order Markov model assigns to each activity $x$ a distribution over its successors, so the parameter is a row-stochastic $K\times K$ matrix consisting of one categorical distribution per row, and the rows are estimated independently. For multinomial observations the maximum-likelihood estimate of a row is the vector of relative frequencies $N(x,y)/N_x$.

\emph{Smoothing.} Relative frequencies are unusable under a logarithmic loss. A successor never observed after $x$ receives probability zero, and a single occurrence of that transition in the evaluation data makes the loss infinite; rows supported on few observations are unstable for the same reason. Adding a constant $\alpha>0$ to every cell removes both defects and has an exact Bayesian reading. If each row carries an independent symmetric Dirichlet$(\alpha,\ldots,\alpha)$ prior and the successors of $x$ are multinomial, the posterior for row $x$ is Dirichlet$(N(x,\cdot)+\alpha)$, and its posterior predictive distribution is the estimator \eqref{eq:model} below. We take $\alpha=1/2$, the Jeffreys prior for the multinomial, which is also the Krichevsky--Trofimov estimator of universal coding \cite{krichevsky1981}.

\emph{The evaluation.} Performance is the expected log probability that the model assigns to transitions drawn from a reference distribution $q$ on pairs $(x,y)$. As a loss this is the cross entropy of $q$ relative to the model, reported in nats, the unit of information belonging to base $e$, one nat being $1/\log 2\approx1.443$ bits, which is the natural scale here because every increment below is a difference of logarithms of count ratios and so carries no conversion constant. Two properties of this convention matter here. The reference $q$ is a choice and not a property of the data, since the same held-out cases admit several reference distributions, and Section~\ref{sec:weighting} shows that the choice reorders the candidates. Log loss is also the only common evaluation under which an assigned probability of zero is inadmissible rather than merely poor, which is what makes the smoothing constant part of the model.

\emph{The quantity.} Data valuation asks what a single training record is worth to a specified learner and task. The established answers are coalitional: Data Shapley averages a record's marginal contribution over subsets of the training set \cite{ghorbani2019}, and influence functions approximate the effect of removing it \cite{koh2017}. Both require retraining or approximation, so the target of any prediction is itself known only up to noise. The quantity studied here is the forward difference instead. Hold the archive fixed, add the transition counts of one trace, and record the change in the reference-weighted log likelihood. Exact computability is the reason for choosing a first-order learner. It removes valuation error as an explanation for whatever the selection experiments show, at the cost of a model that represents no long-range dependence.

\emph{The experiment in outline.} Four disjoint groups of cases are used and should not be confused. The \emph{archive} is the set of cases the learner has already counted; it is held fixed throughout. The \emph{candidate pool} is the set of cases that might be acquired, of which an acquisition rule may take a limited number. A \emph{calibration reference} is a set of held-out cases used to compute the increments that a scoring rule is allowed to see, and a \emph{test reference} is a disjoint set used only to report results. Cases are assigned to these roles at random, and the assignment is repeated five times; we call one such assignment a \emph{role allocation} and report the spread across the five, since a single allocation reveals nothing about stability. A \emph{budget} caps what an acquisition rule may take, either in whole cases or in recorded transitions, and the two units are not interchangeable.

\section{Objects and valuation target}\label{sec:objects}
\subsection{From traces to transition counts}
Let $d=(a_1,\ldots,a_L)$ be a trace with $L\geq2$. Its count matrix is
\begin{equation}
 C_d(x,y)=\sum_{t=1}^{L-1}\ind\{a_t=x,\ a_{t+1}=y\},
 \qquad T_d=L-1=\sum_{x,y}C_d(x,y).
\end{equation}
For an archive $B$, write $N=\sum_{d\in B}C_d$ and $N_x=\sum_yN(x,y)$. The row for $x$ records which activities followed $x$. With $K$ possible output labels and a smoothing constant $\alpha>0$, the predictor is
\begin{equation}\label{eq:model}
 p_N(y\mid x)=\frac{N(x,y)+\alpha}{N_x+K\alpha}.
\end{equation}
As described in Section~\ref{sec:background}, the $\alpha$ pseudo-counts keep every transition at positive probability and make \eqref{eq:model} the posterior predictive distribution of a row-wise Dirichlet prior. We use the estimator without assuming that the underlying workflow is a stationary Markov process.

\subsection{Reference distributions}\label{sec:references}
A reference distribution $q(x,y)\geq0$ with $\sum_{x,y}q(x,y)=1$ describes the transitions on which the model is evaluated. Define
\begin{equation}\label{eq:utility}
 U_q(N)=\sum_{x,y}q(x,y)\log p_N(y\mid x),\qquad
 \ell_q(N)=-U_q(N).
\end{equation}
Higher $U_q$ means lower log loss. Logarithms are natural, so the unit is a nat; one millinat is $10^{-3}$ nats. For a nonempty finite reference set $R$, two weightings are natural:
\begin{align}
 q_{\mathrm c}(x,y)&=\frac1{|R|}\sum_{e\in R}\frac{C_e(x,y)}{T_e},\label{eq:case}\\
 q_{\mathrm t}(x,y)&=\frac{\sum_{e\in R}C_e(x,y)}{\sum_{e\in R}T_e}.\label{eq:transition}
\end{align}
Equation~\eqref{eq:case} samples a case uniformly and then a transition within it. Equation~\eqref{eq:transition} samples uniformly from all recorded transitions. If $R$ contains one two-transition case and one ten-transition case, the two carry weights $1/2$ and $1/2$ under $q_{\mathrm c}$ but $1/6$ and $5/6$ under $q_{\mathrm t}$. Neither objective is intrinsically preferable. The choice records whether errors matter per case or per recorded step, and Section~\ref{sec:weighting} quantifies its effect.

\section{An exact expression for individual value}
Hold $N$, $q$, and $\alpha$ fixed. The marginal increment is
\begin{equation}\label{eq:value}
 v_q(d\mid N)=U_q(N+C_d)-U_q(N).
\end{equation}
A positive value means that adding $d$ reduces the reference log loss. This is a fixed-archive increment, not an average over coalitions of training examples as in Data Shapley \cite{ghorbani2019}.

Write $A_{xy}=N(x,y)+\alpha$, $A_x=N_x+K\alpha$, $c_{xy}=C_d(x,y)$, $c_x=\sum_yc_{xy}$, and $q_x=\sum_yq(x,y)$. The ratio of updated to original predictions is
\begin{equation}
 \frac{p_{N+C_d}(y\mid x)}{p_N(y\mid x)}
 =\frac{A_{xy}+c_{xy}}{A_{xy}}\cdot\frac{A_x}{A_x+c_x}
 =\frac{1+c_{xy}/A_{xy}}{1+c_x/A_x}.
\end{equation}
Taking logarithms and averaging under $q$ gives the exact increment
\begin{equation}\label{eq:exact}
 v_q(d\mid N)=\sum_{x,y}q(x,y)\log\!\left(1+\frac{c_{xy}}{A_{xy}}\right)
 -\sum_xq_x\log\!\left(1+\frac{c_x}{A_x}\right).
\end{equation}
The first term rewards counts added to particular transitions and is large when a transition matters under $q$ but is rare in the archive. The second accounts for row normalization: counts added to one outcome reduce the probability assigned to the others in that row. The two terms have opposite signs and their balance determines the sign of the increment. Additional data can increase the loss when the new counts move predictions away from the reference.

\subsection{Divergence form and the attainable gain}
Equation~\eqref{eq:exact} is convenient for computation but uninformative about magnitude. A second form serves that purpose. For every row $x$ with $q_x>0$ write $\rho_x(y)=q(x,y)/q_x$ for the reference conditional, $H(\rho_x)$ for its entropy, and $\bar H_q=\sum_xq_xH(\rho_x)$.

Throughout, $v_q(C\mid N)=U_q(N+C)-U_q(N)$ also denotes the increment of an arbitrary nonnegative count matrix $C$, of which \eqref{eq:value} is the case $C=C_d$. Sums over rows are understood to run over $\{x:q_x>0\}$, on which $\rho_x$ is defined.

\begin{proposition}[Divergence form]\label{prop:kl}
For every nonnegative archive $N$ and every reference $q$,
\begin{equation}\label{eq:lossdecomp}
 \ell_q(N)=\bar H_q+\sum_{x}q_x\,\KL\!\left(\rho_x\,\|\,p_N(\cdot\mid x)\right),
\end{equation}
and for every nonnegative count matrix $C$,
\begin{equation}\label{eq:vkl}
 v_q(C\mid N)=\sum_xq_x\Big[\KL\!\left(\rho_x\,\|\,p_N(\cdot\mid x)\right)-\KL\!\left(\rho_x\,\|\,p_{N+C}(\cdot\mid x)\right)\Big].
\end{equation}
\end{proposition}
\begin{proof}
Group the sum in \eqref{eq:utility} by rows: $\ell_q(N)=-\sum_xq_x\sum_y\rho_x(y)\log p_N(y\mid x)$. Adding and subtracting $\sum_y\rho_x(y)\log\rho_x(y)$ inside each row gives \eqref{eq:lossdecomp}. Equation~\eqref{eq:vkl} is the difference of two instances of \eqref{eq:lossdecomp}, in which the entropy term cancels because it does not depend on $N$.
\end{proof}

A trace therefore helps exactly to the extent that it moves the row predictors toward the reference conditionals in the $q$-weighted average. Since divergences are nonnegative, the decomposition also limits what acquisition can achieve.

\begin{corollary}[Attainable gain]\label{cor:ceiling}
Let $\Lambda_q(N)=\sum_xq_x\KL(\rho_x\|p_N(\cdot\mid x))$. Then $v_q(C\mid N)\leq\Lambda_q(N)$ for every nonnegative $C$, and
$\sup_{C}v_q(C\mid N)=\Lambda_q(N)$, the supremum being over nonnegative integer count matrices. It is attained only in the exceptional case that some admissible $C$ makes $p_{N+C}(\cdot\mid x)=\rho_x$ for every row with $q_x>0$.
\end{corollary}
\begin{proof}
The inequality follows from \eqref{eq:vkl} and $\KL\geq0$, with equality only when the second divergence vanishes in every row of positive weight. For the supremum, take $C_M(x,y)=\lfloor M\rho_x(y)\rfloor$ on rows with $q_x>0$ and zero elsewhere. Its row sums satisfy $M-K<c_x\leq M$, so $p_{N+C_M}(y\mid x)\to\rho_x(y)$ as $M\to\infty$. Because smoothing keeps every $p_{N+C_M}(y\mid x)$ bounded away from zero for fixed $M$, and the limit is approached uniformly on the finite label set, $\KL(\rho_x\|p_{N+C_M}(\cdot\mid x))\to0$; hence $v_q(C_M\mid N)\to\Lambda_q(N)$.
\end{proof}

Write $\eta_q(S\mid N)=J_q(S\mid N)/\Lambda_q(N)$ for the \emph{acquisition efficiency} of a batch $S$, with $J_q$ the joint gain defined in \eqref{eq:batch}. The efficiency expresses an improvement relative to what was available, which a raw millinat difference does not. The supremum in Corollary~\ref{cor:ceiling} is taken over arbitrary nonnegative count matrices, not over batches that a budget admits or that any set of traces can realize, so $\Lambda_q(N)$ is an upper bound and not the value of an attainable optimum. The ceiling $\Lambda_q(N)$ is also computed from an empirical reference, so it bounds the stated objective on that reference set rather than on the underlying population. Section~\ref{sec:ceilingresults} reports both quantities for BPI 2012, where $\bar H_q$ accounts for three quarters of the initial loss.

\subsection{A two-outcome example}
Fix the two-label alphabet $\{b,c\}$, the active row $x=b$, $\alpha=1/2$, and counts $N(x,b)=0$, $N(x,c)=2$. The initial probabilities under the Dirichlet-smoothed estimator \eqref{eq:model} are $(1/6,5/6)$. A one-transition trace $(x,b)$ changes them to $(3/8,5/8)$. With $q(x,b)=r$ and $q(x,c)=1-r$,
\begin{equation}\label{eq:example}
 v(r)=r\log(9/4)+(1-r)\log(3/4)=r\log3+\log(3/4).
\end{equation}
The trace helps precisely when $r>\log(4/3)/\log3\approx0.262$. At $r=1/2$ its value is $0.2616$ nats; at $r=1/10$ it is $-0.1778$ nats. The same trace and the same archive give opposite answers because the workload differs.

When updates are small relative to the smoothed counts, the first-order Taylor approximation $\log(1+z)\approx z$ turns \eqref{eq:exact} into $\sum_{x,y}q(x,y)c_{xy}/A_{xy}-\sum_xq_xc_x/A_x$, a weighted count added per smoothed count minus its row penalty. The experiments use the exact formula throughout, including the sparse cells where this approximation is poor.

\section{The effect of the reference weighting}\label{sec:weighting}
The two weightings of Section~\ref{sec:references} differ only in how the reference cases are pooled. Because $v_q$ is linear in $q$, the gap between the two increments of the same candidate has a closed form.

\begin{proposition}[Weighting identity]\label{prop:weighting}
Let $R$ be a finite reference set of cases with counts $C_e$ and totals $T_e\geq1$, and let $q_{\mathrm c},q_{\mathrm t}$ be as in \eqref{eq:case} and \eqref{eq:transition}. For a nonnegative count matrix $C$ put $L_C(x,y)=\log p_{N+C}(y\mid x)-\log p_N(y\mid x)$ and define the per-case benefit
\begin{equation}
 w_C(e)=\sum_{x,y}\frac{C_e(x,y)}{T_e}L_C(x,y).
\end{equation}
Then, with empirical mean and covariance taken over $R$ under the uniform distribution,
\begin{equation}\label{eq:weighting}
 v_{q_{\mathrm t}}(C\mid N)-v_{q_{\mathrm c}}(C\mid N)=\frac{\Cov_R(T,w_C)}{\E_R[T]}.
\end{equation}
\end{proposition}
\begin{proof}
Write $\widehat C_e=C_e/T_e$, so that $q_{\mathrm c}=|R|^{-1}\sum_e\widehat C_e$ and $q_{\mathrm t}=\sum_e(T_e/\sum_fT_f)\widehat C_e$. Since $v_q(C\mid N)=\sum_{x,y}q(x,y)L_C(x,y)$ is linear in $q$, and $\sum_{x,y}\widehat C_e(x,y)L_C(x,y)=w_C(e)$,
\begin{equation}
 v_{q_{\mathrm t}}-v_{q_{\mathrm c}}=\sum_{e\in R}\left(\frac{T_e}{\sum_fT_f}-\frac1{|R|}\right)w_C(e)
 =\frac{1}{\sum_fT_f}\sum_{e\in R}\left(T_e-\E_R[T]\right)w_C(e),
\end{equation}
and dividing numerator and denominator by $|R|$ gives \eqref{eq:weighting}.
\end{proof}

The two objectives therefore disagree about a candidate only through the covariance between how long a reference case is and how much the candidate helps it, and coincide exactly where length and benefit are uncorrelated across the reference set. The identity applies to a batch as well, since $J_q$ is also linear in $q$. The gap can have mean near zero across candidates and still spread further than the increments themselves, which is what the data show: in the earlier period the two increments of the same candidate have Spearman correlation $-0.315$, the covariance term has standard deviation $2.23$ millinats against $1.59$ for the case-weighted increment, and $33.1$ per cent of candidates change sign. The Spearman coefficient is the Pearson correlation computed on the ranks of the two scores rather than on the scores themselves, so it measures agreement of the two orderings and is invariant under any increasing transformation of either. The weighting therefore leaves the average value of a trace almost unchanged and alters which traces are valuable.

\section{Descriptor information and its limits}
\subsection{Prediction from restricted information}
For a trace of length $L$ with $m$ distinct activities, $p_a$ the relative frequency of activity $a$, $r$ the number of adjacent repeated activities, and $e$ the number of distinct directed transition pairs, we define the descriptor vector $F(d)\in\R^7$ as
\begin{equation}
 F(d)\coloneqq\left(\log(1+L),\ \log(1+m),\ -\sum_a p_a\log p_a,\
 \frac{r}{L-1},\ \frac{L-m}{L},\ \frac{e}{m},\ \max_a p_a\right).
\end{equation}
Its entries are log length, log distinct count, activity-frequency entropy, immediate-repeat fraction, revisit fraction, distinct edges per distinct activity, and maximum frequency share; no activity identities are retained. Repeated cases remain separate acquisition units even when their sequences coincide.

A regressor $f$ predicts $v_q(d\mid N)$ from $F(d)$. The archive and the reference are held fixed when its training labels are created, so their influence enters through the labels although they are not inputs to $f$. Learning such value predictors has precedents in amortized attribution \cite{covert2024}; here the target labels are exact for a specified empirical reference.

To state the information limit, take a random candidate $D$, set $V=v_q(D\mid N)$, and assume $V\in L^2$. The best squared-error predictor is the conditional expectation $g(F)=\E[V\mid F]$. For any square-integrable $f(F)$, expanding $V-f(F)=(V-g(F))+(g(F)-f(F))$ gives
\begin{equation}\label{eq:projection}
 \E[(V-f(F))^2]=\E[\Var(V\mid F)]+\E[(g(F)-f(F))^2],
\end{equation}
the cross term vanishing because $\E[V-g(F)\mid F]=0$. The first term is information destroyed by the descriptors, the second estimation error above that limit; $(a,b)$ and $(a,c)$ share descriptors but touch different cells, so no deterministic function of $F$ predicts both exactly whenever their increments differ. This standard $L^2$ identity is useful here because its first term is estimable: grouping candidates with identical descriptor signatures gives a median estimated bound on the achievable $R^2$ of $0.951$ in the earlier period, against $0.870$ reached by the fitted forest. The estimate is optimistic in two ways. It resolves only exact descriptor collisions, so candidates with close but unequal signatures contribute nothing to the estimated conditional variance, and the collision groups are small, which biases a within-group variance downward. It is also unstable across role allocations, ranging from $0.891$ to $0.965$ while the fitted $R^2$ ranges from $0.444$ to $0.885$. In every allocation the gap between the fit and the estimated bound exceeds the gap between the bound and one, so the larger share of the prediction error is estimation error rather than information lost by the descriptors, but the margin is not uniform.

Prediction accuracy and selection quality are not the same criterion, and Section~\ref{sec:ladder} shows them ordered oppositely for two regressors on this descriptor vector.

\section{From singleton values to batches}\label{sec:batch}
For a set $S$ of candidate cases the actual batch gain is
\begin{equation}\label{eq:batch}
 J_q(S\mid N)=U_q\!\left(N+\sum_{d\in S}C_d\right)-U_q(N),
\end{equation}
which is generally not $\sum_{d\in S}v_q(d\mid N)$, because each new trace changes the archive against which the others contribute. In the calculation of \eqref{eq:example} with $r=1/2$, two copies of $(x,b)$ give final probabilities $(1/2,1/2)$ and a joint gain of $\tfrac12\log(9/5)=0.2939$ nats against $0.5232$ for twice the singleton value. The opposite sign occurs in the same archive and under the same reference. With $r=1/2$, the candidates $i=(x,b)$ and $j=(x,c)$ have singleton values $0.2616$ and $-0.1194$ nats, which sum to $0.1422$, while their joint gain is $0.2067$ nats. A trace that is harmful on its own improves the batch, because it repairs the row normalization that the other distorts.

\begin{proposition}[Sign of the pairwise interaction]\label{prop:sign}
Extend $U_q$ to nonnegative real count matrices. For candidates with count matrices $C_i,C_j$, row sums $c_{ix}=\sum_yC_i(x,y)$, and any nonnegative $M$,
\begin{equation}\label{eq:mixed}
 D^2U_q(M)[C_i,C_j]
 =-\sum_{x,y}\frac{q(x,y)C_i(x,y)C_j(x,y)}{(M(x,y)+\alpha)^2}
 +\sum_x\frac{q_x c_{ix}c_{jx}}{(M_x+K\alpha)^2}.
\end{equation}
\begin{enumerate}[label=(\alph*)]
\item If $C_i$ and $C_j$ share no cell but share a row $x$ with $q_x>0$, then $D^2U_q(M)[C_i,C_j]>0$: the candidates are complements and the joint gain exceeds the sum of the individual increments.
\item If both add counts to a single common cell $(x,y)$ with $q_x>0$ and to no other cell, then the mixed derivative is nonpositive, so the candidates are substitutes, if and only if $\rho_x(y)\geq p_M(y\mid x)^2$.
\end{enumerate}
Consequently $S\mapsto J_q(S\mid N)$ is in general neither submodular nor supermodular; that is, the gain of adding a candidate to a larger set is neither always smaller nor always larger than the gain of adding it to a smaller one.
\end{proposition}
\begin{proof}
Differentiating $U_q(M)=\sum_{x,y}q(x,y)[\log(M(x,y)+\alpha)-\log(M_x+K\alpha)]$ twice in the directions $C_i,C_j$ gives \eqref{eq:mixed}. For (a) the first sum is empty and the second is strictly positive. For (b) write $c_i=C_i(x,y)$, $c_j=C_j(x,y)$, $A_{xy}=M(x,y)+\alpha$ and $A_x=M_x+K\alpha$; then \eqref{eq:mixed} equals $c_ic_jq_x[A_x^{-2}-\rho_x(y)A_{xy}^{-2}]$, which is nonpositive exactly when $\rho_x(y)\geq(A_{xy}/A_x)^2=p_M(y\mid x)^2$. The two worked examples above realize both signs for one archive and one reference, so no single modularity direction can hold in general.
\end{proof}

Part (b) states that two traces reinforcing the same transition are substitutes when the reference conditional mass on that transition is at least the square of the probability the model currently assigns to it. Since $p\in[0,1]$ and therefore $p^2\leq p$, the condition is weak, and substitution is the typical case for the transitions on which a selection rule concentrates. Part (a) accounts for the opposite outcome, in which a batch is worth more than the sum of its parts. Neither effect is visible in a ranking of singleton scores, and the failure of submodularity matters because the standard approximation guarantee for greedy maximization assumes it \cite{nemhauser1978}. Accumulated over a batch, the second-order terms are bounded.

\begin{proposition}[A bound for nonadditivity]\label{prop:bound}
Fix a nonnegative archive $N$, a probability distribution $q$, and $\alpha>0$. For nonnegative candidate count matrices $C_i$ with row sums $c_{ix}$, define
\begin{equation}\label{eq:bij}
 b_{ij}=\sum_{x,y}\frac{q(x,y)C_i(x,y)C_j(x,y)}{(N(x,y)+\alpha)^2}
 +\sum_x\frac{q_x c_{ix}c_{jx}}{(N_x+K\alpha)^2}.
\end{equation}
Then for every finite candidate set $S$,
\begin{equation}\label{eq:interaction}
 \left|J_q(S\mid N)-\sum_{i\in S}v_q(i\mid N)\right|
 \leq\sum_{i<j,\ i,j\in S}b_{ij}.
\end{equation}
\end{proposition}
\begin{proof}
Whenever $M\geq N$ entrywise, the triangle inequality applied to \eqref{eq:mixed} bounds $|D^2U_q(M)[C_i,C_j]|$ by $b_{ij}$. Order $S$ arbitrarily. The joint gain telescopes into successive marginal increments. For the $j$th addition, compare its increment at $N+\sum_{i<j}C_i$ with its increment at $N$ and integrate the mixed derivative over the unit square, with one direction $\sum_{i<j}C_i$ and the other $C_j$. Bilinearity and the derivative bound give an absolute difference of at most $\sum_{i<j}b_{ij}$. Summing over additions proves the claim.
\end{proof}

The bound holds for a fixed batch under either budget rule. It is not an optimality or regret guarantee for ranking by individual scores under a transition constraint, and small smoothed counts make it loose; Section~\ref{sec:interactionresults} reports how loose. The general gap between attribution and subset selection is discussed in \cite{wang2024}.

\section{Data and experimental design}\label{sec:design}
\subsection{Sources, roles, and evaluation}
BPI Challenge 2012 records loan-application activity at a Dutch financial institution \cite{dongen2012}. The released log has 13,087 cases and 262,200 events. We retain the 164,506 events marked \texttt{COMPLETE}, spanning 23 activity names. Events with equal timestamps keep their source order. Cases crossing the 60th percentile of case start times are removed: 7,097 cases remain in the earlier period, 5,235 in the later period, and 755 are purged.

Whole cases receive disjoint roles. Of the earlier cases, 1,419 form a base pool, 2,129 train the value regressor, 2,129 are acquisition candidates, and 710 each supply calibration and test references. Later cases split into 2,617 calibration and 2,618 test cases. The vocabulary is taken from the whole base pool with one extra unknown label, so $K=24$; the predictor's counts use only the first 128 base cases. The smoothing constant is $\alpha=1/2$.

Labels for the value regressor use the earlier calibration reference. Two regressors are fitted to the same seven descriptors and the same labels: a random forest, an average of 100 regression trees fitted to bootstrap resamples with maximum depth 8 and minimum leaf size 5, and a ridge regression with penalty 1 on standardized inputs. Both are frozen after fitting and reused for both periods. They differ only in functional form, so a difference between them is a property of the fitted regressor and not of the information the descriptors carry. Test references supply the reported losses and the true increments, never training labels. Each reference weighting gets its own calibration labels and its own pair of regressors, so the weighting comparison evaluates corresponding pipelines rather than one ranking under two metrics.

\subsection{Acquisition budgets}
At a case budget of 20 per cent, a ranking selects its first $\lfloor0.2\cdot2129\rfloor=425$ cases. At a transition threshold of 20 per cent, it selects the shortest prefix whose cumulative transition count reaches 20 per cent of the candidate pool's transitions; the final whole case may exceed the threshold. Random selection uses uniform subsets for case budgets and uniform random permutations with the same stopping rule for transition thresholds.

The distinction changes the comparison. At the primary allocation under case weighting, the transition rule selects 1,162 forest-ranked cases containing 4,381 transitions against 425.8 cases and 4,389.6 transitions for the random rule: matched transition volume, very different case counts. At a fixed 425-case budget, forest selection instead receives 1,355 transitions against 4,357.9. The two units encode different cost assumptions, and neither is a correction of the other.

The budget restricts additions to the count model and excludes the reference data, the training of the regressor, and the inspection of candidate descriptors. The main comparison uses 199 random batches at the primary allocation and 39 at each of four further allocations, all reusing the same log and temporal boundary, so their ranges measure sensitivity to role assignment rather than uncertainty across organizations. Shortest-first and fewest-distinct-activities rankings serve as inexpensive baselines.

\note{\textbf{Analysis status.} Case weighting with a case budget was specified in a locally dated analysis plan; it was not externally registered. The transition-budget and transition-weighted analyses were added after seeing the initial results, as were the ceiling, interaction, and sequential-rule analyses reported in Sections~\ref{sec:ceilingresults} to~\ref{sec:interactionresults}. The present focus on BPI is likewise retrospective. Both regressors were fitted in the same run and neither was dropped, but the decision to report the ridge comparison prominently was made after its results were seen, so the contrast of Section~\ref{sec:ladder} is an observation to be replicated under a regressor fixed in advance, not a test.}

\section{Results}
\subsection{Selection under the two weightings and budgets}\label{sec:selectionresults}
Under case weighting and the transition threshold, the forest batch at the primary allocation improves log loss over the mean random batch by 65.2 millinats in the earlier period and 26.0 in the later period; the gains over the unchanged archive are 171.6 and 106.4 millinats for the forest and for random selection in the earlier period. The comparison is less decisive against simple rules. Shortest-first gains 50.4 millinats over random in the earlier period and 13.8 in the later period, fewest-distinct 43.2 and 0.7, and across five allocations shortest-first has an earlier-period median advantage of 39.2 millinats against 39.6 for the forest. The favorable example does not establish that the forest beats an inexpensive length rule.

Two controls bound the alternative reading that the transition threshold rewards any rule preferring short cases. Ranking by descending length selects 115 cases and loses 216.0 millinats to random selection in the earlier period; a forest fitted to randomly permuted labels selects 222 cases and loses 244.9. Conversely, more cases do not imply a larger gain: at a matched transition volume shortest-first takes the most cases of any rule considered, a median of 1{,}328 against 1{,}138 for the forest and 1{,}266 for the ridge regression, and gains the least of the three. The case count is therefore not a sufficient statistic for the outcome.

\begin{table}[tb]
\centering
\small\setlength{\tabcolsep}{4pt}
\begin{tabular}{@{}llrr@{}}
\toprule
Reference weight & Budget unit & Earlier period & Later period\\
\midrule
Cases & Cases & +9.0 [-6.9, +14.3] & -22.5 [-44.2, -13.8] \\
Cases & Transitions & +39.6 [-16.4, +65.2] & +4.7 [-26.8, +36.5] \\
Transitions & Cases & -84.5 [-126.2, -63.1] & -56.1 [-82.4, -29.4] \\
Transitions & Transitions & -64.0 [-84.0, -58.0] & -41.9 [-48.0, -30.8] \\
\bottomrule
\end{tabular}

\caption{Forest gain over random in millinats: median [minimum, maximum] over five role allocations at the 20 per cent acquisition setting.}\label{tab:robustness}
\end{table}

Figure~\ref{fig:selection} shows the spread across allocations. Under transition weighting both descriptor regressors lose to random selection in all five allocations at both budget units. Section~\ref{sec:ceilingresults} shows how much was available to be won in each configuration, and Section~\ref{sec:ladder} separates the contribution of the fitted score from that of the selection principle.

\subsection{Attainable gain and acquisition efficiency}\label{sec:ceilingresults}
Corollary~\ref{cor:ceiling} makes the magnitude of these differences measurable. At the primary allocation under case weighting, the initial loss on the earlier test reference is $\ell_q(N_0)=0.944$ nats, of which the reference's own conditional entropy $\bar H_q=0.704$ nats cannot be removed by any amount of data. The attainable gain is $\Lambda_q(N_0)=0.240$ nats, a quarter of the loss. Measured against that ceiling, the 65.2 millinat advantage of the forest over random selection is 27 per cent. These three figures are for the primary allocation; Table~\ref{tab:ceiling} reports medians over all five, where the ceiling is 230 millinats.

\begin{table}[tb]
\centering
\footnotesize\setlength{\tabcolsep}{4pt}
\begin{tabular}{@{}llrrrrrrrrr@{}}
\toprule
& & \multicolumn{3}{c}{Decomposition of the initial loss} & \multicolumn{6}{c}{Share of the ceiling realized (\%)}\\
\cmidrule(lr){3-5}\cmidrule(lr){6-11}
Weight & Period & $\ell_q(N_0)$ & $\bar H_q$ & $\Lambda_q$ & Random & Forest & Ridge & Seq. & Oracle & All\\
\midrule
Cases & Earlier & 0.928 & 0.698 & 230 & 43 & 61 & 69 & 82 & 86 & 55 \\
Cases & Later & 0.957 & 0.753 & 204 & 48 & 54 & 65 & 81 & 85 & 61 \\
Transitions & Earlier & 1.156 & 0.991 & 165 & 70 & 30 & 30 & 70 & 79 & 89 \\
Transitions & Later & 1.168 & 1.002 & 166 & 71 & 44 & 47 & 71 & 79 & 90 \\
\bottomrule
\end{tabular}

\caption{Loss decomposition \eqref{eq:lossdecomp} and acquisition efficiency at the 20 per cent transition budget, medians over five role allocations; $\ell_q(N_0)$ and $\bar H_q$ in nats, $\Lambda_q$ in millinats. The last six columns give $\eta_q$ in per cent for random selection, the two descriptor regressors, the sequential calibration rule, the sequential test-score oracle, and the whole candidate pool added without any budget. $\Lambda_q$ bounds arbitrary count additions, not budgeted ones, so these shares are not efficiencies against an attainable optimum.}\label{tab:ceiling}
\end{table}

Table~\ref{tab:ceiling} reports the decomposition for all four configurations. Two observations follow.

The first concerns the room available to any selection rule. Conditional entropy accounts for 75 per cent of the loss under case weighting and 86 per cent under transition weighting, which leaves median ceilings of 230 and 165 millinats. Under case weighting random selection realizes 43 per cent of that ceiling, leaving most of it open to a rule; under transition weighting it already realizes 70 per cent of a smaller ceiling, leaving little. This is a property of the objective rather than of any estimator.

The second concerns unrestricted acquisition. Under case weighting, adding the whole candidate pool of 21{,}936 transitions without any budget realizes 55 per cent of the ceiling in the earlier period, less than the 61 per cent that the forest reaches with a fifth of them. Under transition weighting the ordering reverses and the unrestricted pool realizes 89 per cent. Equation~\eqref{eq:vkl} accounts for this. Under case weighting a median of 35 per cent of candidates have a negative increment and move the row predictors away from the reference conditionals, against 4 per cent under transition weighting. The pool figure also shows that $\Lambda_q$ is not approached by the data available here: no budget-free addition of the entire pool comes within 40 per cent of it under case weighting.

\subsection{A comparison of selection rules}\label{sec:ladder}
The descriptor forest is one point on a scale from random selection to a rule that sees the evaluation reference itself. Table~\ref{tab:ladder} places six further rules on that scale, at the same budget and against the same random comparators. Static rules rank candidates once; sequential rules recompute the exact increment of every remaining candidate after each addition and take the largest increment per transition, which costs about $2.4$ million exact marginal evaluations for one weighting, scorer, and allocation.

\begin{table}[tb]
\centering
\footnotesize\setlength{\tabcolsep}{3pt}
\begin{tabular}{@{}lrrrr@{}}
\toprule
& \multicolumn{2}{c}{Case weighting} & \multicolumn{2}{c}{Transition weighting}\\
\cmidrule(lr){2-3}\cmidrule(lr){4-5}
Selection rule & Earlier & Later & Earlier & Later\\
\midrule
Descriptor forest & +40 [-16, +65] & +5 [-27, +36] & -64 [-84, -58] & -42 [-48, -31] \\
Descriptor ridge & +60 [+42, +79] & +30 [+19, +51] & -65 [-75, -52] & -38 [-45, -31] \\
Shortest first & +39 [+7, +50] & +10 [-25, +19] & -160 [-216, -147] & -169 [-229, -163] \\
Calibration score, static & +57 [+7, +72] & +28 [-10, +41] & -64 [-96, -57] & -41 [-61, -40] \\
Calibration score, sequential & +93 [+72, +95] & +63 [+55, +68] & +1 [-4, +2] & +3 [-2, +4] \\
Test-score oracle, sequential & +102 [+84, +105] & +72 [+64, +76] & +14 [+13, +17] & +15 [+13, +16] \\
\midrule
Whole pool, no budget & +27 [+20, +30] & +27 [+20, +30] & +32 [+29, +33] & +32 [+31, +32] \\
\bottomrule
\end{tabular}

\caption{Gain over the mean random batch in whole millinats at the 20 per cent transition budget: median [minimum, maximum] over five role allocations. Calibration-scored rules use earlier-period calibration data only; the oracle uses the test reference of the period it is evaluated on and is not implementable. The last row respects no budget.}\label{tab:ladder}
\end{table}
\begin{figure}[tb]
\centering
\includegraphics[width=\linewidth]{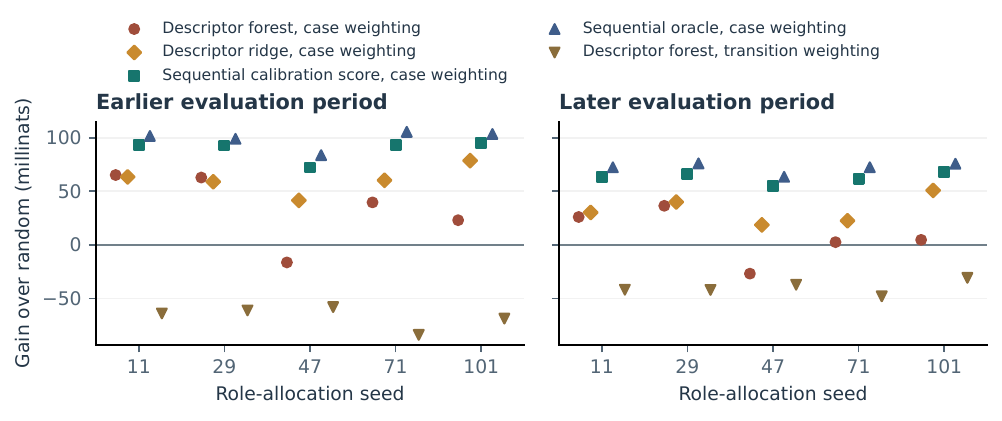}
\caption{BPI 2012 at the 20 per cent transition budget. Each point is one role allocation; positive values favor the rule over the corresponding mean random batch. The descriptor forest crosses zero under case weighting and is uniformly negative under transition weighting. The ridge regression on the same descriptors and the sequential calibration rule stay positive in every allocation, the latter within 13 millinats of the oracle.}\label{fig:selection}
\end{figure}

Under case weighting, the instability of Table~\ref{tab:robustness} belongs to the fitted forest and not to the descriptors. A ridge regression on the same seven descriptors, the same labels, and the same training cases is positive in all ten allocation-period cells and exceeds shortest-first at the median in both periods, neither of which the forest does. The forest collapses at one allocation, where its fitted $R^2$ falls to $0.444$ against a median of $0.870$; the ridge fit has no comparable outlier. Replacing one estimator of the same target by another therefore changes the qualitative conclusion, which is a reason to fix the estimator in advance rather than evidence that either form is preferable.

Accuracy on individual increments does not order the rules by the quality of the batches they select. The forest predicts single-trace values far better than the ridge regression in both periods, median $R^2$ $0.870$ against $0.623$ and median Spearman correlation $0.927$ against $0.726$ in the earlier period, yet realizes the smaller share of the ceiling in Table~\ref{tab:ceiling}. The two batches behave alike at the primary allocation, concentrating comparable mass in their five largest cells and realizing comparable fractions of the sum of their singleton scores. What differs is the variance of the fitted score across allocations.

Exact increments close most of the remaining gap. The sequential calibration rule is positive in all five allocations and both periods, and the oracle adds a further 9 millinats at the median and at most 13 in any allocation, so the implementable rule reaches 82 per cent of the ceiling against 86 for the oracle.

Under transition weighting, the objective leaves little room for any rule. The oracle gains only 14.4 and 14.9 millinats in the two periods, the sequential calibration rule is indistinguishable from random selection, and both descriptor regressors are negative, so the finding does not depend on the choice between them. The negative entries of Table~\ref{tab:robustness} show descriptor scores misranking candidates for an objective whose weighting differs from the one their labels were computed under, as quantified by Proposition~\ref{prop:weighting}, while little is available to win even with perfect information.

The sequential rules require the candidate's own count matrix and therefore assume that a candidate can be inspected before it is acquired, which is the case the descriptor route is intended to avoid; Table~\ref{tab:ladder} compares selection quality, not total cost.

\subsection{Batch interaction}\label{sec:interactionresults}
Proposition~\ref{prop:sign} predicts that concentrated batches lose value to substitution. The measurement confirms this and gives its magnitude. At the primary allocation under case weighting and the transition budget, the forest batch of 1,162 cases has a joint gain of 171.6 millinats against a sum of individual increments of 2,516.4 millinats, which is 6.8 per cent of the sum of the singleton scores. One random comparator batch of 400 cases and comparable transition volume realizes 29.1 per cent, or 90.9 against 312.5 millinats; the 106.4 millinats quoted in Section~\ref{sec:selectionresults} is the mean over the 199 random batches, not this one. The medians over five allocations are 5.6 and 27.6 per cent. The ridge batch behaves like the forest batch, at 7.1 per cent.

The mechanism is visible in the counts. The forest batch spreads its 4,381 added transitions over 85 distinct cells with 73.8 per cent of the mass in the five largest, while the random batch covers 117 cells with 36.6 per cent in the five largest. Both descriptor rules select many short, similar traces, and by Proposition~\ref{prop:sign}(b) these are substitutes wherever the reference conditional mass exceeds the squared current probability, which holds for the cells they concentrate on. Selection by singleton score is therefore systematically optimistic, and the two descriptor rules are equally affected despite their different ranking accuracy, which is consistent with the interaction being a property of the concentration of the batch rather than of the score that produced it. The bound of Proposition~\ref{prop:bound} is valid but uninformative at this scale, since it permits 80.0 nats of nonadditivity against the 2.34 observed. Its content is qualitative. The departure from additivity is governed by pairwise products of counts weighted by inverse squared smoothed counts, which is why a batch concentrated on few sparse cells departs furthest.

\subsection{Stability across evaluation periods}\label{sec:transfer}
A ranking computed on one reference is then used on another. Since $v_q(C\mid N)=\sum_{x,y}q(x,y)L_C(x,y)$ is linear in the reference and $q-q'$ sums to zero, subtracting the midrange of $L_C$ and applying H\"older's inequality gives $|v_q(C\mid N)-v_{q'}(C\mid N)|\leq\TV(q,q')\,\osc(L_C)$, where $\osc$ is the range of $L_C$ over cells; the same argument applied to $L_{C_i}-L_{C_j}$ shows that a pair keeps its order whenever its margin exceeds $\TV(q,q')\,\osc(L_{C_i}-L_{C_j})$. At the total variation of $0.0807$ between the two test references under case weighting, the bound certifies none of the 2,128 adjacent pairs of the earlier-period ranking, because sparse cells make $\osc$ large. The realized instability is much smaller than the bound admits, in that 400 adjacent pairs, or 18.8 per cent, actually reverse. In this setting the transfer across periods has to be measured rather than bounded.

\section{Interpretation and limitations}
Under a case-weighted objective and a transition constraint, selection by trace value is worth roughly a quarter of the attainable gain over random selection, and most of that is reachable with exact increments computed on calibration data. The result does not extend to descriptor scores as a class. The fitted forest is unstable across role allocations and not clearly better than shortest-first, while a ridge regression on the same descriptors is positive in all ten allocation-period cells and beats shortest-first by about 20 millinats at the median, despite predicting individual increments considerably less accurately. The forest result therefore does not refute the descriptor route, and neither descriptor result should be relied on, because the choice between two estimators fitted in the same run changes the sign of the conclusion and was made after the outcomes were known. A replication should fix one regressor in advance and compare four arms, namely random selection, shortest-first, that regressor, and exact sequential scoring.

At the primary allocation none of the $B=199$ random transition-budget batches matches the forest batch in either period, so the Monte Carlo tail probability $p_{\mathrm{MC}}=(1+\#\{b:J_q(S_b\mid N)\geq J_q(S_{\mathrm{forest}}\mid N)\})/(B+1)$ equals $0.005$, the smallest attainable value at this simulation size. In this randomization test the null distribution is generated by the random-selection procedure rather than assumed, and the added unit in numerator and denominator keeps the test valid at finite $B$, so the value describes an extreme result relative to that procedure, conditional on the split, the candidate pool, and the test reference. Holm adjustment \cite{holm1979}, a step-down correction that controls the probability of any false rejection across a family of tests, raises it to 0.020 within the family of four dataset-period comparisons that the analysis plan specified. That correction does not account for the subsequent choice of a favorable weighting, budget unit, or case study, and none of the four originally planned case-budget comparisons across both logs passed the planned threshold. The configuration is therefore treated as exploratory.

The learner uses only the immediately preceding activity, so it models no long-range dependence, predicts no loan decision, and optimizes no interactive agent; agent-trajectory curation studies different objectives and models \cite{zheng2026}. The transition count is one acquisition cost among several, since privacy review, storage, and labeling carry case-level costs, and the sequential rules carry a scoring cost that the budget does not charge them.

Three sources of error act at once. An estimated score can differ from the exact increment, a singleton score computed exactly on a calibration reference can differ from the increment on a test reference, and even exact test-reference scores need not select the best joint batch, because of the interactions of Section~\ref{sec:interactionresults}. Moving from static to sequential exact scoring in Table~\ref{tab:ladder} isolates the last of these, worth 36 millinats in the earlier period.

The five allocations overlap, use different numbers of random comparator batches, and come from one historical process, so they bound sensitivity to role assignment and nothing wider. The ceiling inherits the sampling error of the empirical reference it is computed from, and no interval is reported for it. A replication should also account for scoring costs and assess the final-case overshoot of the transition rule.

\section{Conclusion}
The exact increment separates the trace, the archive, and the reference distribution, and its divergence form bounds what any acquisition can achieve. On BPI 2012 that bound is a quarter of the initial loss. Random selection realizes just under half of it, and a sequential rule using exact calibration increments realizes most of it while remaining positive across role allocations. Descriptor ranking captures part of the same gain, and how much depends on which regressor is fitted to the same descriptors and labels, with the more accurate predictor of individual increments selecting the worse batches. Individual-score correlation is therefore insufficient evidence in either direction. A batch must be judged under the intended objective and acquisition rule, and against the gain that was available.

\renewcommand{\refname}{References}
{\footnotesize
\bibliographystyle{alphaurl}
\bibliography{references}
}
\end{document}